\documentclass[10pt,a4paper]{article}
\usepackage[utf8]{inputenc}

\usepackage{amsmath,amssymb,amsthm,amsfonts, soul}
\usepackage{dsfont}
\usepackage[langfont=typewriter,funcfont=slant]{complexity}
\usepackage[left=1in,bottom=1in,top=1in,right=1in]{geometry}
\usepackage[colorlinks=true]{hyperref}
\usepackage{framed}
\usepackage{thmtools,thm-restate}
\usepackage{csquotes}
\usepackage[nameinlink, noabbrev, capitalize]{cleveref}
\usepackage{booktabs}       %
\usepackage{nicefrac}       %
\usepackage{microtype} 
\usepackage{natbib}
\usepackage{multirow}

\usepackage[disable]{todonotes} %
\presetkeys{todonotes}{inline}{}

\makeatletter
\if@todonotes@disabled

\else

\fi
\makeatother

\usepackage{commath} %

\usepackage{framed}
\usepackage[ruled,noline]{algorithm2e}

\theoremstyle{definition}
\newtheorem{definition}{Definition}[section]

\theoremstyle{remark}

\theoremstyle{plain}
\newtheorem{theorem}{Theorem}[section]

\theoremstyle{plain}

\theoremstyle{plain}

\theoremstyle{plain}
\newtheorem{lemma}[theorem]{Lemma}

\newtheorem{openquestion}{Open Question}

\newclass{\NPH}{NPH}

\newcommand{\eps}{\varepsilon}
\newcommand{\Ind}{\mathds{1}}

\newcommand{\Secref}[1]{\hyperref[#1]{Section \ref*{#1}}}
\newcommand{\Appref}[1]{\hyperref[#1]{Appendix \ref*{#1}}}

\crefname{equation}{}{}
\crefrangeformat{equation}{(#3#1#4) --~(#5#2#6)}
\crefmultiformat{equation}{(#2#1#3)}{, (#2#1#3)}{, (#2#1#3)}{, (#2#1#3)}
\crefname{lemma}{Lemma}{Lemmas}
\crefname{section}{Section}{Sections}
\crefname{subsubsubsection}{Section}{Sections}
\crefname{remark}{Remark}{Remarks}
\crefname{figure}{Figure}{Figures}
\crefname{table}{Table}{Tables}
\Crefname{lemma}{Lemma}{Lemmas}
\crefname{theorem}{Theorem}{Theorems}
\Crefname{theorem}{Theorem}{Theorems}

\DeclareMathOperator*{\Ex}{\mathbb{E}}

\DeclareMathOperator*{\maj}{maj}
\newcommand{\randERM}{E_{\textrm{rand}}}
\DeclareMathOperator*{\outd}{out-degree}
\newcommand{\OSrand}{\mathcal{O}^\star_{\textrm{rand}}}

\newcommand{\adv}{\mathtt{Adv}}
\newcommand{\obl}{\mathtt{ObliAdv}}
\newcommand{\err}{\mathrm{err}}

\title{\textbf{Learning with Monotone Adversarial Corruptions}}
 \author{ \textbf{Kasper Green Larsen} \\ Aarhus University \and \textbf{Chirag Pabbaraju} \\ Stanford University \and \textbf{Abhishek Shetty} \\ MIT   }

\date{}
\begin{document}
\maketitle
\begin{abstract}
    We study the extent to which standard machine learning algorithms rely on exchangeability and independence of data by introducing a monotone adversarial corruption model. In this model, an adversary, upon looking at a ``clean'' i.i.d. dataset, inserts additional ``corrupted'' points of their choice into the dataset. These added points are constrained to be monotone corruptions, in that they get labeled according to the ground-truth target function. Perhaps surprisingly, we demonstrate that in this setting, all known optimal learning algorithms for binary classification can be made to achieve suboptimal expected error on a new independent test point drawn from the same distribution as the clean dataset. On the other hand, we show that uniform convergence-based algorithms do not degrade in their guarantees. 
    Our results showcase how optimal learning algorithms break down in the face of seemingly helpful monotone corruptions, exposing their overreliance on exchangeability. 
\end{abstract}

\section{Introduction}
\label{sec:intro}
Understanding the structures that allow for generalization from seen to unseen data is perhaps the core pursuit of learning theory, and the fact that it proposes approaches addressing this accounts for both its theoretical and practical significance. 
A key assumption at the heart of the most basic results on generalization is that of independence, or the closely related assumption of exchangeability in the data. Given this, a natural question to ask is: %
how central is this assumption to the problem of generalization? %
This question has been studied in various guises such as in the context of domain adaptation, transfer learning, and more generally in the context of robustness.
In this paper, we explore this question by studying a minimal model that violates exchangeability, while maintaining the ``quality'' of data, allowing us to isolate the role of these assumptions in generalization.

One motivation for the model comes from trying to understand the modern machine learning dogma of ``more data is better'', even when the additional data is not directly representative of the target distribution. 
Further, machine learning practice often involves sophisticated data curation techniques that adaptively select data points based on preliminary analysis of the data.
Though these practices have shown great empirical success, their theoretical understanding is still limited. 
As we will see, these can break exchangeability in the data; in addition, standard techniques in machine learning do not suffice to show that such practices do not hurt generalization, even when the additional data is supposedly ``clean''.
The key point that we want to convey using our model is that this seemingly innocuous dependence in the data can in fact provably break several approaches to generalization such as leave-one-out and ensemble methods. 
On the other hand, we show that arguments based on uniform convergence are in fact robust to these perturbations, and perhaps %
hint at the success of loss minimization in modern machine learning.

To model this setting, we introduce the \textit{monotone adversary} model (\Cref{def:monotone-adversary}), where a dataset is entirely labeled honestly, but may comprise of both representative as well as non\ -representative data. Here, a dataset $S$ is prepared as follows: suppose the target ground-truth is some hypothesis $h^\star$ belonging to a class of hypotheses $\mathcal{H}$, and the representative marginal distribution of the data is $\mathcal{D}$. First, $n$ \textit{clean} points are drawn i.i.d.\ from $\mathcal{D}$ and added to $S$. Thereafter, a so-called ``monotone adversary'', who has complete knowledge about $h^\star, \mathcal{H}$ as well as $\mathcal{D}$, looks at these $n$ points, and adds $m$ \textit{corrupted} points of their choice to $S$. The combined dataset of $n + m$ points is then entirely labeled by $h^\star$, shuffled randomly, and presented to the learning algorithm. The objective of the learning algorithm is to output a hypothesis $h$ that has low error with respect to the representative data distribution $\mathcal{D}$.

We focus on the case when the hypothesis class $\mathcal{H}$ comprises of binary hypotheses, although the model can be defined more generally.
In this case, the learnability of $\mathcal{H}$ is governed by the VC dimension $d$ \citep{vapnik1971uniform} of the class. 
In particular, the optimal expected error in the standard setting with $n$ i.i.d.\ points is $\Theta(d/n)$ \cite{haussler1994predicting,ehrenfeucht1989general}.  
Naturally, this is the benchmark to try and achieve in the presence of a monotone adversary.

\subsection{Our Contributions}
\label{sec:model}

Our first main result shows that a natural perspective on generalization given by the leave-one-out principle, which is at the heart of several learning algorithms, can be completely broken by a monotone adversary. 
An instantiation in the binary classification setting is the celebrated \textit{One-inclusion Graph} (OIG) algorithm \citep{haussler1994predicting}, which attains the optimal expected error (up to a factor of 2 \citep{li2001one}) in the standard setting: we show that this algorithm can be forced to suffer \textit{constant error} under monotone adversarial corruptions, even for learning a class of VC dimension 1. 
Our result shows that exchangeability is indeed necessary in the strongest possible sense for the guarantees of the algorithm. 

\begin{theorem}[Leave-one-out/OIG Lower Bound (Informal)]
    \label{thm:oig-lb-informal}
    There exists a monotone adversary setting for learning a binary hypothesis class of VC dimension $1$ with $n$ clean points and $n$ corrupted points where the OIG algorithm suffers expected error $1/4$.
\end{theorem}

Another reason as to why the OIG lower bound above is significant is that recent work has brought to light several natural learning settings, such as partial concept classes \citep{alon2022theory} and multiclass classes \citep{brukhim2022characterization}, where the OIG algorithm is the only known way to obtain a learner that attains any vanishing expected error at all.
Thus, our result suggests that learnability in these settings might be fragile to strong violation of exchangeability in the data.

A different class of optimal learning algorithms in the binary setting are \textit{ensemble/voting based} algorithms. 
These algorithms construct carefully chosen subsets of the training dataset, obtain a hypothesis from the underlying class $\mathcal{H}$ for each subset that is consistent with that subset (i.e., an \textit{Empirical Risk Minimizer} for the subset), and use the majority vote of these hypotheses to make their final predictions. 
This class includes the first optimal learning algorithm by \cite{hanneke2016optimal}, as well as the more recent optimal algorithms based on Bagging \citep{larsen2023bagging} and Majority-of-Three \citep{aden2024majority}. 
Our next result shows that any majority voting algorithm can be made to be suboptimal %
in the monotone adversary model.

\begin{theorem}[Majority Voting Lower Bound (Informal)]
    \label{thm:majority-voting-lb-informal}
    For any majority voting algorithm that uses subsets of size at least $t$, there exists a monotone adversary setting for learning a binary hypothesis class of VC dimension $d$ with $n$ clean points and $2n/t$ corrupted points where the algorithm suffers expected error $\Omega(d\log(n/d)/n)$.
\end{theorem}

We note that all the three optimal majority voting algorithms mentioned above satisfy $t = \Omega(n)$, and hence the lower bound holds for these algorithms with just a \textit{constant} number of corrupted samples. \Cref{thm:oig-lb-informal}, in conjunction with \Cref{thm:majority-voting-lb-informal}, establishes the fragility %
of all known techniques for achieving optimal algorithms, to monotone adversarial corruptions.

A key point to note here is that the lower bound above does not prove that the majority voting algorithm does not have a vanishing error rate; rather, it just shows the suboptimality of the rate. 
In fact, on the positive side, this suboptimality occurs just due to the fact that the algorithms are using empirical risk minimization. 
That is, we show that the ERM principle is robust to monotone adversarial corruptions. 
The intuition here is that however the adversary may try to thwart the learning algorithm, they cannot make the ground-truth hypothesis $h^\star$ look bad on the training dataset by virtue of not being able to corrupt the labels. 
By exploiting this observation and instantiating standard uniform convergence, we can show that in any monotone adversary model with $n$ clean points and an \textit{arbitrary} number of corrupted points, \textit{every} Empirical Risk Minimizer (ERM) on the training dataset (comprising of both the clean and corrupted points) attains expected error $O(d\log(n/d)/n)$. 
  Therefore, the suboptimality that an adversary can enforce in the monotone adversray model is at most a $\log(n/d)$ factor. 

\begin{theorem}[ERM Upper Bound (Informal)]
    \label{thm:erm-ub-informal}
    In any monotone adversary setting for learning a binary hypothesis class of VC dimension $d$ with $n$ clean points and an arbitrary number of corrupted points, every ERM attains expected error $O(d\log(n/d)/n)$.
\end{theorem}

We note that the upper bound above is the best that we can hope for ERMs, %
since there exists a matching lower bound for ERMs in the setting where there are no corrupted points at all \citep{haussler1994predicting,auer2007new}.
The key takeaway from these results is that even though the ERM principle might not be optimal, it naturally comes with guarantees that are robust to certain forms of misspecification in the data, and perhaps hints at the success of loss minimization, as opposed to more sophisticated techniques, in modern machine learning practice.

Finally, to emphasize the point that lack of exchangeability is indeed the key obstacle in obtaining optimal expected error, we show that if the monotone adversary is constrained to be \textit{oblivious}, i.e., the corrupted points are specified independently of the clean points, we can recover optimal expected error $O(d/n)$ using the OIG algorithm.

\begin{theorem}[Oblivious Adversary OIG Upper Bound (Informal)]
    In any oblivious monotone adversary setting for learning a binary class of VC dimension $d$ with $n$ clean points and an arbitrary number of corrupted points, the OIG algorithm attains expected error $O(d/n)$.
\end{theorem}

Our results are summarized in \Cref{table:summary}.

\begin{table}[t]
\centering
\begin{tabular}{c|c|c|c|}
\cline{2-4}
\multicolumn{1}{c|}{} & Algorithm & \# Corrupted Points & Error \\ \hline

\multicolumn{1}{|c|}{\multirow{4}{*}{Lower Bounds}} &
\multirow{2}{*}{One-inclusion Graph} &
\multirow{2}{*}{$n$} &
$\Omega(1)$ \\
\multicolumn{1}{|c|}{} & & & (\Cref{thm:oig-lower-bound}) \\
\cline{2-4}
\multicolumn{1}{|c|}{} &
\multirow{2}{*}{Majority Voting over $t$-sized subsets} &
\multirow{2}{*}{$2n/t$} &
$\Omega(d\log(n/d)/n)$ \\
\multicolumn{1}{|c|}{} & & & (\Cref{thm:majority-voter-lb-random-erm}) \\ \hline

\multicolumn{1}{|c|}{\multirow{4}{*}{Upper Bounds}} &
\multirow{2}{*}{ERM} &
\multirow{2}{*}{Arbitrary} &
$O(d\log(n/d)/n)$ \\
\multicolumn{1}{|c|}{} & & & (\Cref{thm:monotone-adversary-erm-sample-complexity}) \\
\cline{2-4}
\multicolumn{1}{|c|}{} &
\multirow{2}{*}{One-inclusion Graph} &
\multirow{2}{*}{Arbitrary (oblivious)} &
$O(d/n)$ \\
\multicolumn{1}{|c|}{} & & & (\Cref{thm:oblivious-monotone-adversary-oig-sample-complexity}) \\ \hline

\end{tabular}
\caption{Summary of our results in the monotone adversary model with $n$ clean points.}
\label{table:summary}
\end{table}

\subsection{Future Directions}
\label{sec:open-questions}

Before proceeding to formally describe all our results, we will outline a host of intriguing open questions that remain unsolved from our study of the monotone adversary model .

The foremost question that is still open is with regards to the optimal expected error that can be achieved for learning a class of binary hypotheses of VC dimension $d$ from $n$ clean points in the presence of a monotone adversary. While the upper bound above for ERM (\Cref{thm:erm-ub-informal}) shows that an expected error of $O(d\log(n/d)/n)$ can always be achieved, we have not been able to obtain any algorithm that shaves the log factor. Similarly, while \Cref{thm:oig-lb-informal} and \Cref{thm:majority-voting-lb-informal} show lower bounds for popular optimal learners, we have not been able to show a general lower bound that precludes all learners from getting error $O(d/n)$.

\begin{openquestion}
    \label{open:optimal-error}
    What is the optimal expected error that can be achieved in the monotone adversary model for learning a binary hypothesis class of VC dimension $d$ with $n$ clean points?
\end{openquestion}

A good place to start here is the setting where the number of corrupted points is small (like constant or $\Theta(d)$). The primary obstacle that we have faced in our numerous efforts towards resolving this question is the lack of a clean technical tool that allows us to deal with the \textit{lack of exchangeability} in the dataset. A central ingredient in the analysis of optimal majority voting algorithms is that even if one subset of the dataset results in a poor ERM, it is unlikely that several ERMs trained on disjoint subsets of the data are simultaneously \textit{all} bad, simply by independence of the different subsets. The adversarial corruptions preclude this attractive property by potentially correlating disjoint subsets of the dataset. %
It is tempting to try to reduce from an adaptive to an oblivious monotone adversary by using subsampling as in \cite{blanc2025adaptive}; however, such reductions typically incur polynomial sample size blowups, whereas we can't even afford a log factor. It is also not clear how techniques like \textit{stable sample compression schemes} \citep{bousquet2020proper} may be made to work, since conditioning on any candidate compression set renders the rest of the samples to be not i.i.d.\ We believe there might be mileage in trying to apply the notion of a \textit{randomized} stable compression scheme \citep{da2024boosting} for an optimal expected error bound, but despite several attempts, have so far been unable to succeed. Proving a general lower against arbitrary learners also appears challenging; it is worth noting that for all the lower bound instances we construct, simple learning rules like taking the majority of all hypotheses in the class attain zero expected error.

Moving beyond the question of optimal expected error for VC classes, one might ask: what can one say about Littlestone classes? The Littlestone dimension $d_L$ of a binary class characterizes \textit{online learnability} of binary classes \citep{littlestone1988learning,ben2009agnostic}, and is at least the VC dimension $d$ of the class. A standard online-to-batch analysis (Chapter 5.10 in \cite{blum2020foundations}) allows converting any online learner with mistake bound $M$ to a learner that has expected error $O(M/n)$ in the setting with $n$ i.i.d.\ examples with no corruptions. Since the Standard Optimal Algorithm (SOA) has mistake bound $d_L$ for binary classes having Littlestone dimension $d_L$, this gives a way to obtain a learner with expected error $O(d_L/n)$, which might be better than the $O(d\log(n/d)/n)$ guarantee of ERM in certain cases. Unfortunately, standard online-to-batch analyses also seem to go awry in the presence of monotone corrupted samples.

\begin{openquestion}
    \label{open:littlestone}
    Is it possible to obtain an expected error $O(d_L/n)$ for learning a binary hypothesis class having Littlestone dimension $d_L$ in the monotone adversary model with $n$ clean points?
\end{openquestion}

Finally, one can even go beyond binary hypothesis classes, and consider the settings of \textit{partial binary classes} \citep{alon2022theory} and \textit{multiclass classes} \citep{brukhim2022characterization}. Even the broader question of obtaining a learner than attains \textit{any vanishing error rate} as $n$ gets large is open in these settings for the monotone adversary model. Namely, since the uniform convergence principle ceases to hold in these settings \citep{daniely2015multiclass,alon2022theory}, the ERM guarantee for binary classes no longer applies. In the standard i.i.d.\ setting with only clean points, the learning algorithm that is the primary workhorse in these settings is the OIG algorithm. Since we are unable to analyze this algorithm at all in the (adaptive) montone adversary setting even in the binary case, it is not clear how to show a guarantee for this algorithm in the partial or multiclass settings. 

\begin{openquestion}
    \label{open:partial-multiclass}
    Is it possible to obtain any expected error that goes to zero as $n$ grows, for learning any learnable partial binary class or multiclass class in the monotone adversary model with $n$ clean points?
\end{openquestion}

A final direction to explore is computational. 
Though we showed that ERM achieves non-trivial learning rates (in the bounded VC dimension setting), there are natural learning problems such as learning convex bodies over Gaussian space and learning monotone functions over the uniform distribution on the hypercube where ERM does not achieve a non-vacuous rate (due to unbounded VC dimension). Nevertheless, computationally efficient learning algorithms exist in these settings, though they rely on relevant distributional assumptions. 
In such settings, it is not clear whether computationally efficient learning algorithms can be developed for monotone adversaries. 

\begin{openquestion}
    \label{open:distribution-dependent}
    Can we develop techniques to handle monotone adversaries in distribution-dependent computational learning settings?
\end{openquestion}

It appears to us that the resolution of the open questions above would involve novel technical tools beyond those available in the standard toolkit.

\subsection{Related Work}
\label{sec:related-work}

The study of models in statistical learning beyond independence is a rich area with a long line of work, which is beyond the scope of this paper to survey in its entirety. 
Perhaps the strongest model considered is online learning or mistake-bounded learning \citep{littlestone1988learning,cesa2006prediction} where the data is assumed to be arbitrary. 
Unfortunately, under such weak assumptions, learnability is characterized by notions such as Littlestone dimension \citep{littlestone1988learning} or sequential Rademacher complexities \citep{rakhlin2015online}, which tend to be significantly larger than their counterparts in the statistical case. 
Though there has been a surge of work towards understanding relaxations on the arbitrariness of data \citep{haghtalab2022smoothed, haghtalaboracle, block2024performance, shetty2024learning, montasser2025beyond, goel2023adversarial, goel2024tolerant}, these models don't directly capture the monotone adversary problem that we consider. 
Perhaps the most important difference is that they focus on regret while we focus on test error on the clean distribution. 
Another related line of work is transfer learning \citep{hanneke2024more}  which studies how learning under one distribution transfers to another distribution. 
Though this line of work gives bounds that transfer from one distribution to another, to the best of our knowledge, it does not give meaningful guarantees in the monotone adversary model that we consider. 
Perhaps the most closely related model and an inspiration for our work is that of \cite{goel2023adversarial}, where they study an online setting with arbitrary corruptions injected in, but where the learner is allowed to abstain from making a prediction. 
Our data generating process is essentially an offline analogue of their setting in the unknown distribution case, where the statistical rate remains open. 
A similar offline analog is also studied by \cite{blum2021robust, gao2021learningcertificationinstancetargetedpoisoning,hanneke2022optimal,balcan2022robustly,balcan2023reliable,chornomaz2025agnosticlearningtargetedpoisoning}, where unlike in our case, the adversary's corruptions are allowed to depend on the test point as well.
This makes the problem significantly more challenging, allowing for lower bounds to established. 

The choice of the name ``monotone adversary" is from the long line of work on semirandom models in theoretical computer science and learning theory \citep{Feige_2021,awasthi2017clusteringsemirandommixturesgaussians}, where they were presented as a means of understanding the robustness of algorithms to small perturbations of the data assumptions.
A salient example of this is the study of recovery thresholds for community detection where an adversary is allowed to add edges within the community and remove edges across communities \citep{moitra2016robustreconstructionthresholdscommunity}.
Perhaps surprisingly, such changes, which one would expect to make the recovery problem easier, make the recovery problem harder. 
Our model can be seen as an analogue to this phenomenon in the context of learning, where the fact that the ``adversarial'' points are consistently labeled should make the learning problem easier, but as we show, this is not the case.
More broadly, our work relates to the theme of understanding learning and statistical inference under robustness to adversarial perturbations which has seen a recent surge of interest through a computational lens \citep{diakonikolas2023algorithmic}.

\section{Preliminaries} 
\label{sec:prelims} 

\subsection{The Monotone Adversary}
\label{sec:monotone-adversary}

\begin{definition}[Monotone Adversary] 
    \label{def:monotone-adversary}
    Let $\mathcal{D}$ be a distribution over $\mathcal{X}$. Let $\mathcal{H}$ be a hypothesis class known to the learner, and let $h^\star \in \mathcal{H}$ be the target hypothesis. For $n \ge 1, m \ge 0$, let $\mathcal{A}=\mathcal{A}(\mathcal{D}, h^\star):\mathcal{X}^n \to \mathcal{X}^m$ be a (possibly randomized) \textit{monotone adversary} that has knowledge of $\mathcal{D}$, $\mathcal{H}$ and the target hypothesis $h^\star$. We denote by $\adv_{\mathcal{D}, h^\star, \mathcal{A}}(n, m)$ the output of the following random process: first, $x_1, \dots , x_{n} \sim \mathcal{D} $ are drawn i.i.d.\ from $\mathcal{D}$: these are the \textit{clean} samples. The adversary then takes $x_1, \dots x_{ n}$ as input and generates $\mathcal{A}(x_1,\dots,x_n)=\tilde{x}_1,\dots,\tilde{x}_m$: these are the \textit{monotone corrupted} samples. The final output is a uniformly random permutation of $((x_1, h^\star(x_1)),\dots,(x_n,h^\star(x_n)), (\tilde{x}_1,h^\star(\tilde{x}_1)),\dots,(\tilde{x}_m, h^\star(\tilde{x}_m)))$.
\end{definition}

Given as input a sample $S \sim \adv_{\mathcal{D}, h^\star, \mathcal{A}}(n, m)$, the aim of a learner is to output a hypothesis $h$ that minimizes the expected error on a new test point drawn from $\mathcal{D}$, i.e.,
\begin{align}
    \err_{\mathcal{D}, h^\star}(h_S) := \Ex_{x \sim \mathcal{D}}[\Ind[h(x) \neq h^{\star}(x)]] = \Pr_{x \sim \mathcal{D}}[h(x) \neq h^{\star}(x)]. 
\end{align}

We note some salient features of this model. First, the entire dataset comprises of only \textit{honest labels}. %
Second, even if the adversary can compose the corrupted points as a fully adaptive function of the clean points, the benchmark for the learning algorithm is its error on a new test point drawn at random from the representative data distribution; in this sense, the test point is not in control of the adversary. 
Third, the introduction of the honestly labeled corrupted points nevertheless completely breaks independence and exchangeability in the data. 
In particular, if we condition on the corrupted points, the distribution of the clean points is no longer i.i.d.\ from the representative distribution.

\subsubsection{Oblivious Monotone Adversary}
\label{sec:oblivious}

We also consider a more benign version of the monotone adversary that still has complete knowledge of $h^\star, \mathcal{H}$ and $\mathcal{D}$, but cannot look at the clean points while preparing the corrupted points---we term this adversary an \textit{oblivious} monotone adversary, and will explicitly qualify it thus to distinguish it from the more powerful adaptive adversary described above.

\begin{definition}[Oblivious Monotone Adversary] 
    \label{def:oblivious-monotone-adversary}
    Let $\mathcal{D}$ be a distribution over $\mathcal{X}$. Let $\mathcal{H}$ be a hypothesis class known to the learner, and let $h^\star \in \mathcal{H}$ be the target hypothesis. For $m \ge 0$, let $\mathcal{A}=\mathcal{A}(\mathcal{D}, h^\star): \emptyset \to \mathcal{X}^m$ be a (possibly randomized) \textit{oblivious monotone adversary} that has knowledge of $\mathcal{D}$, $\mathcal{H}$ and the target hypothesis $h^\star$. We denote by $\obl_{\mathcal{D}, h^\star, \mathcal{A}}(n, m)$ the output of the following random process: first, $x_1, \dots , x_{n} \sim \mathcal{D} $ are drawn i.i.d.\ from $\mathcal{D}$. The adversary $\mathcal{A}$ generates $\tilde{x}_1,\dots,\tilde{x}_m$ independently without seeing $x_1, \dots , x_{n} $. The final output is a uniformly random permutation of $((x_1, h^\star(x_1)),\dots,(x_n,h^\star(x_n)), (\tilde{x}_1,h^\star(\tilde{x}_1)),\dots,(\tilde{x}_m, h^\star(\tilde{x}_m)))$.
\end{definition}

\subsection{Learning Algorithms}
\label{sec:learners}

The most natural learning algorithm is an Empirical Risk Minimizer.
\begin{definition}[Empirical Risk Minimizer]
    \label{def:erm}
    An Empirical Risk Minimizer (ERM) is a learning algorithm that takes as input a training set $S = ((x_1,h^\star(x_1)),\dots,(x_N,h^\star(x_N)))$ and produces a hypothesis $h_S \in \mathcal{H}$ with $h_S(x_i)=h^\star(x_i)$ for all $(x_i, h^\star(x_i)) \in S$.
\end{definition}

The following is a textbook result for the ERM algorithm.

\begin{theorem}[ERM Sample Complexity (Chapter 5.6 in \cite{blum2020foundations})]
    \label{thm:erm-sample-complexity}
    Let $\mathcal{H}$ be any binary hypothesis class over $\mathcal{X}$ having VC dimension $d$. Let $\mathcal{D}$ be any distribution over $\mathcal{X}$, and let $h^\star \in \mathcal{H}$ be the target hypothesis. With probability at least $1-\delta$ over the draw of $S=((x_1,y_1),\dots,(x_N, y_N))$, where $x_1,\dots,x_N$ are drawn i.i.d.\ from $\mathcal{D}$ and $y_i=h^\star(x_i)$ for every $i$, it holds that every $h_S \in \mathcal{H}$ that is an ERM with respect to $S$ satisfies
    \begin{align}
        \err_{\mathcal{D}, h^\star}(h_S) \le 100\frac{d\log(N/d)+\log(1/\delta)}{N}.
    \end{align}
\end{theorem}

A majority voter uses ERM as a base learning algorithm in order to output a majority vote.

\begin{definition}[Majority Voter]
    \label{def:majority-voter}
    A majority voter is a learning algorithm that takes as input an ERM algorithm and a training set $S = ((x_1,h^\star(x_1)),\dots,(x_N,h^\star(x_N)))$. As a function of $N$ alone, it then produces a list $L=L_1,\dots,L_k$, where each $L_i$ is a sequence of indices $\ell^i_1,\dots,\ell^i_{t_i} \in \{1,\dots,N\}$. It then constructs the training sets $S_i = ((x_{\ell^i_1}, h^\star(x_{\ell^i_1})),\dots,(x_{\ell^i_{t_i}}, h^\star(x_{\ell^i_{t_i}})))$ corresponding to the indices in $L_i$. Finally, it runs the ERM algorithm on each $S_i$ to produce hypotheses $h_1,\dots,h_k$ and outputs the final classifier $h(x) = \maj(h_1(x),\dots,h_k(x))$ where $\maj(\cdot)$ denotes a majority vote. We say that the majority voter uses sub-samples of size $t$ if each list $L_i$ has at least $t$ distinct indices.
\end{definition}

We observe that previous optimal learners given by \cite{hanneke2016optimal}, Bagging \citep{larsen2023bagging} as well as Majority-of-Three \citep{aden2024majority} all fall in the category of majority voters, with Hanneke's algorithm having $k=N^{\log_4 3}$ and sub-samples of size at least $N/2$, Bagging having $k=O(\log(N/\delta))$ and sub-samples of size $\Omega(N)$ (with probability $1-\exp(-\Omega(N))$), and Majority-of-Three having $k=3$ and sub-samples of size $N/3$. We also note that the analysis of previous majority voters crucially needs sub-samples of size $\Omega(N)$.

We will also discuss a qualitatively different learning algorithm, known as the One-inclusion Graph algorithm.
\begin{definition}[One-inclusion Graph \citep{alon1987partitioning,haussler1994predicting}]
    \label{def:oig}
    The One-inclusion Graph (OIG) algorithm takes as input a training set $S = ((x_1,h^\star(x_1)),\dots,(x_N,h^\star(x_N)))$, and outputs a hypothesis $h$ (not necessarily in $\mathcal{H}$). For any $x \in \mathcal{X}$, $h(x)$ is obtained as follows. First, the algorithm constructs a graph $\mathcal{G}$, whose vertex set $V$ is the set of projections (distinct labelings) by members of $\mathcal{H}$ onto $(x_1,\dots,x_N,x)$, so that every vertex can be identified by a unique pattern in $\{0,1\}^{N+1}$. A vertex $u$ connects to vertex $v$ by an edge in the ``direction'' $i \in [N+1]$ if  $u_{i} \neq v_i$, and $u_j = v_j$ for every $j \neq i$. An orientation $\sigma$ of the edges in $\mathcal{G}$ maps every edge $e$ in the graph to one of the two vertices it is connected to. The out-degree of a vertex $u$ in the orientation $\sigma$ is the number of edges connected to it which $\sigma$ maps to the other end-point of the edge (which can be at most $N+1$, one for every direction $i \in [N+1]$). The OIG algorithm constructs any orientation $\sigma$ of $\mathcal{G}$ which minimizes the largest out-degree of any vertex in the graph\footnote{There may be multiple orientations that minimize the out-degree; it suffices to consider any of these. In this sense, the one-inclusion graph algorithm is really a \textit{class} of algorithms.}. Consider potentially the two vertices of the form $(h^\star(x_1),\dots,h^\star(x_N),0)$ (the ``0 vertex'') and $(h^\star(x_1),\dots,h^\star(x_N),1)$ (the ``1 vertex'') in $\mathcal{G}$ (at least one of these vertices exists because $h^\star \in \mathcal{H}$). If only the 0 vertex exists, set $h(x)=0$; otherwise, if only the 1 vertex exists, set $h(x)=1$. Otherwise, if both the vertices exist, consider the edge $e$ in direction $N+1$ connecting them. Set $h(x)=0$ if $\sigma$ orients $e$ towards the 0 vertex, and $h(x)=1$ if it orients it towards the 1 vertex.
\end{definition}

The following is a well-known structural result about orientations for one-inclusion graphs that are induced by hypothesis classes of VC dimension $d$.

\begin{theorem}[Bounded out-degree of OIG \citep{haussler1994predicting}]
    \label{thm:oig-outdegree}
    Let $\mathcal{H}$ be a binary hypothesis class over $\mathcal{X}$ having VC dimension $d$. Then, for any $n \ge 1$ and $S=(x_1,\dots,x_n)$, there exists an orientation of the one-inclusion graph of the projection of $\mathcal{H}$ onto $S$, such that the out-degree of every vertex in the orientation is at most $d$.
\end{theorem}

\section{Upper Bounds}
\label{sec:upper-bounds}

First, we observe that the optimal expected error that a learning algorithm can achieve upon receiving a training dataset drawn from $\adv_{\mathcal{D}, h^\star, \mathcal{A}}(n,m)$, where $h^\star$ belongs to a hypothesis class $\mathcal{H}$ having VC dimension $d$, is $O(d/n)$. This follows from the $\Omega(d/n)$ lower bound on the expected error that any learning algorithm must suffer, even in the case when there are no corruptions, and the input solely comprises of $n$ i.i.d.\ draws from $\mathcal{D}$ \citep{ehrenfeucht1989general}.

The following result shows that the ERM algorithm achieves an expected error rate of $O(d\log(n/d)/n)$, even in the presence of a monotone adversary.

\begin{theorem}[Monotone Adversary ERM Upper Bound]
    \label{thm:monotone-adversary-erm-sample-complexity}
    Let $\mathcal{H}$ be any binary hypothesis class over $\mathcal{X}$ having VC dimension $d$. Let $\mathcal{D}$ be any distribution over $\mathcal{X}$, and let $h^\star \in \mathcal{H}$ be the target hypothesis. Let $\mathcal{A}$ be any monotone adversary. With probability at least $1-\delta$ over the draw of $S \sim \adv_{\mathcal{D}, h^\star, \mathcal{A}}(n, m)$, it holds that every $h_S \in \mathcal{H}$ that is an ERM with respect to $S$ satisfies
    \begin{align}
        \err_{\mathcal{D}, h^\star}(h_S) \le 100\frac{d\log(n/d)+\log(1/\delta)}{n}.
    \end{align}
\end{theorem}
\begin{proof}
    Let $\Pi$ denote the uniformly random permutation that acts on the clean+corrupted data before being fed to the learner as input. For a sample $S \sim \adv_{\mathcal{D}, h^\star, \mathcal{A}}(n, m)$, upon conditioning on $\Pi$, the clean data points exist as the data points $S' = (x_{i_1},y_{i_1}),\dots,(x_{i_n}, y_{i_n})$ in $S$, for some fixed distinct indices $i_1,\dots,i_n$, and the corrupted dataset $S''$ corresponds to $S \setminus S'$. Furthermore, under this conditioning, $x_{i_1},\dots,x_{i_n}$ are i.i.d.\ draws from $\mathcal{D}$, with $y_{i_j}=h^\star(x_{i_j})$ for every $j \in [n]$ (note that we are not conditioning on the corrupted data points). Now, observe that any ERM $h_S$ with respect to all of $S$ is also an ERM with respect to $S'$. This is because the labels on the corrupted points are still given by $h^\star$, and hence any ERM with respect to $S$ must label all points in $S$ according to $h^\star$, including $S'$. Letting $\eps := 100\frac{d\log(n/d)+\log(1/\delta)}{n}$, we thus have that
    \begin{align*}
        &\Pr_{S \sim \adv_{\mathcal{D}, h^\star, \mathcal{A}}(n, m)}\left[\text{$\exists$ ERM $h_S$ s.t. } \err_{\mathcal{D}, h^\star}(h_S) > \varepsilon\right] = \Ex_{\Pi}\left[\Pr_{S', S''}\left[\text{$\exists$ ERM $h_S$ s.t. } \err_{\mathcal{D}, h^\star}(h_S) > \varepsilon ~\Big|~ \Pi\right]\right] \\
        &\le \Ex_{\Pi}\left[\Pr_{S', S''}\left[\text{$\exists$ ERM $h_{S'}$ s.t. } \err_{\mathcal{D}, h^\star}(h_{S'}) > \varepsilon ~\Big|~ \Pi\right]\right] = \Ex_{\Pi}\left[\Pr_{S'}\left[\text{$\exists$ ERM $h_{S'}$ s.t. } \err_{\mathcal{D}, h^\star}(h_{S'}) > \varepsilon ~\Big|~ \Pi\right]\right] \\
        &\le \delta.
    \end{align*}
    The first inequality follows from our above reasoning, which implies that if there exists an ERM with respect to $S$ that has large error, then there exists an ERM with respect to $S'$ that has large error. The concluding inequality follows from \Cref{thm:erm-sample-complexity}, together with the fact that the conditional distribution of data points in $S'$ is i.i.d.\ from $\mathcal{D}$ and labeled by $h^\star$.
\end{proof}

Note that the expected error rate in \Cref{thm:monotone-adversary-erm-sample-complexity} is optimal for ERM, even in the weaker case of oblivious adversaries. In particular, there exist hypothesis classes, where for any ERM, there exists a data distribution $\mathcal{D}$ for which the ERM, trained on $n$ i.i.d.\ samples from $\mathcal{D}$ \textit{without} any corruptions, suffers expected error $\Omega\left(\frac{d\log(n/d)}{n}\right)$ \citep{haussler1994predicting,auer2007new}. 

Nevertheless, for oblivious adversaries, it is possible to get an improved error rate using different algorithms; in particular, the following theorem shows that the OIG algorithm gets an expected error of $O(d/n)$. This is essentially because the corrupted points added by an oblivious adversary do not depend on the clean points, and hence the distribution of the clean points is i.i.d.\ from $\mathcal{D}$, even when we \textit{condition} on the corrupted points. Note that this is not at all the case in the setting of an adaptive adversary that adds corrupted points as a function of the clean points.

\begin{theorem}[Oblivious Monotone Adversary OIG Upper Bound]
    \label{thm:oblivious-monotone-adversary-oig-sample-complexity}
    Let $\mathcal{H}$ be any binary hypothesis class over $\mathcal{X}$ having VC dimension $d$. Let $\mathcal{D}$ be any distribution over $\mathcal{X}$, and let $h^\star \in \mathcal{H}$ be the target hypothesis. Let $\mathcal{A}$ be any oblivious monotone adversary. Then, for $S \sim \obl_{\mathcal{D}, h^\star, \mathcal{A}}(n, m)$, the hypothesis $h_S$ ouput by the OIG algorithm on $S$ satisfies
    \begin{align}
        \Ex_{S \sim \obl_{\mathcal{D}, h^\star, \mathcal{A}}(n, m)}\left[\err_{\mathcal{D}, h^\star}(h_S)\right] \le \frac{d}{n+1}.
    \end{align}
\end{theorem}
\begin{proof}
    Let us again condition on the uniformly random permutation $\Pi$ that acts on the clean+corrupted data before it is given to the learner. For a sample $S \sim \obl_{\mathcal{D}, h^\star, \mathcal{A}}(n, m)$, upon conditioning on $\Pi$, the clean data points exist as the data points $S' = (x_{i_1},y_{i_1}),\dots,(x_{i_n}, y_{i_n})$ in $S$, for some fixed distinct indices $i_1,\dots,i_n$, and the corrupted dataset $S''$ corresponds to $S \setminus S'$. Let us further condition on the corrupted points in $S''$. Since $S''$ was specified by an oblivious adversary, it is independent of $S'$. Thus, conditioned on both $\Pi$ and $S''$, the distribution of $S'$ is that of $n$ i.i.d.\ draws from $\mathcal{D}$ labeled by $h^\star$. We have that
    \begin{align*}
        \Ex_{S \sim \obl_{\mathcal{D}, h^\star, \mathcal{A}}(n, m)}\left[\err_{\mathcal{D}, h^\star}(h_S)\right] = 
        \Ex_{\Pi, S''}\left[\Ex_{S'}\left[\err_{\mathcal{D}, h^\star}(h_S) ~\Big|~ \Pi, S''\right]\right].
    \end{align*}
    By our reasoning above, and writing out the definition of $\err_{\mathcal{D}, h^\star}(h_S)$, the inner expectation is equal to
    \begin{align*}
        \Ex_{S'}\left[\err_{\mathcal{D}, h^\star}(h_S) ~\Big|~ \Pi, S''\right] %
        &= \Ex_{\substack{x_{i_1},\dots,x_{i_n} \sim \mathcal{D}^n \\ S' = ((x_{i_1}, h^\star(x_{i_1})),\dots,(x_{i_n}, h^\star(x_{i_n})))}}\left[\Ex_{x \sim \mathcal{D}}\left[ \Ind[h_S(x) \neq h^\star(x)]~\Big|~\Pi, S''\right]\right]. 
    \end{align*}
    Now, let $i_{n+1}$ stand for a separate index, e.g., $i_{n+1}=n+m+1$. Let $j \sim [n+1]$ be a uniformly random index drawn from $[n+1]$. Then, recalling that $S=(S', S'')$ and using exhangeability, we have that the expectation above is equal to
    \begin{align*}
        &\Ex_{\substack{x_{i_1},\dots,x_{i_{n+1}} \sim \mathcal{D}^{n+1} \\ \bar{S} = ((x_{i_1}, h^\star(x_{i_1})),\dots,(x_{i_{n+1}}, h^\star(x_{i_{n+1}})))}}\left[\Ex_{j \sim [n+1]}\left[ \Ind[h_{(\bar{S}_{-i_j}, S'')}(x_{i_j}) \neq h^\star(x_{i_j})]~\Big|~\Pi, S''\right]\right] \\
        &= \frac{1}{n+1} \Ex_{\substack{x_{i_1},\dots,x_{i_{n+1}} \sim \mathcal{D}^{n+1} \\ \bar{S} = ((x_{i_1}, h^\star(x_{i_1})),\dots,(x_{i_{n+1}}, h^\star(x_{i_{n+1}})))}}\left[\sum_{j=1}^{n+1} \Ind[h_{(\bar{S}_{-i_j}, S'')}(x_{i_j}) \neq h^\star(x_{i_j})]~\Big|~\Pi, S''\right].
    \end{align*}
    In the above, $\bar{S}_{-i_j}$ is shorthand for $\bar{S} \setminus (x_{i_j}, h^\star(x_{i_j}))$. We will argue that for every fixed $\bar{S}$, the summation inside the expectation is at most $d$, which will give us the desired result. Let $S''=(x_{\ell_1},\dots,x_{\ell_m})$, and observe that conditioned on $\Pi$ and $S''$, we can further upper bound the summation as
    \begin{align*}
        \sum_{j=1}^{n+1} \Ind[h_{(\bar{S}_{-i_j}, S'')}(x_{i_j}) \neq h^\star(x_{i_j})] \le \sum_{j=1}^{n+1} \Ind[h_{(\bar{S}_{-i_j}, S'')}(x_{i_j}) \neq h^\star(x_{i_j})] + \sum_{j =1}^{m}\Ind[h_{(\bar{S}, S''_{-\ell_j})}(x_{\ell_j}) \neq h^\star(x_{\ell_j})].
    \end{align*}
    But now observe that in order to make a prediction in any of the summations above, the OIG algorithm constructs the same graph, because the combined set of training+test points is the same, namely $(\bar{S}, S'')$. Furthermore, by definition of the OIG algorithm which predicts on a point according to the orientation of the edge in the direction of that point, the summation above is precisely equal to the out-degree of the vertex that is the projection of $h^\star$ in this graph. From \Cref{thm:oig-outdegree}, we know that the out-degree of every vertex in the orientation that the algorithm constructs is at most $d$. This concludes the proof.
\end{proof}

We remark that the guarantee above holds for any one-inclusion graph algorithm, i.e., no matter how it orients the one-inclusion graph, so long as the orientation minimizes the max out-degree.

\section{Lower Bounds}
\label{sec:lower-bounds}

In this section, we present our lower bounds which show how optimal learning strategies in the standard i.i.d.\ setting may become suboptimal with monotone adversarial corruptions. We will require using the following auxiliary lemma that is based on the coupon collector problem. For completeness, we give a proof in \Cref{sec:aux}.
\begin{lemma}[Coupon Collector]
    \label{lem:coupon}
    Consider a set of $r = c n/\log(n/d)$ elements $x_1,\dots,x_r$, where $c>0$ is a sufficiently large constant, $n \geq cd$ and $d \geq 1$. Then for a set $S$ of $n$ i.i.d.\ uniform samples from $x_1,\dots,x_r$, it holds with probability at least $1/2$ that there are at least $d$ elements $x_i$ not in $S$.
\end{lemma}

\subsection{Lower Bounds against Majority Voting Strategies}
\label{sec:majority-voting-lower-bounds}

Recall that one class of optimal learners in the standard i.i.d.\ setting with no adversary is based on majority voting among ERMs. This includes the algorithms given by \cite{hanneke2016optimal}, Bagging \citep{larsen2023bagging} and Majority-of-Three \citep{aden2024majority}. In this section, we prove a lower bound generally against the class of majority voting strategies, showing that a monotone adversary may force such strategies into obtaining suboptimal error. 

We first prove a lower bound for majority voting invoked with \emph{some} ERM, i.e.,\ where we explicitly design an adversarial ERM strategy that causes suboptimal error. Note that an optimal error bound of $O(d/n)$ holds for the above mentioned voting algorithms \textit{regardless} of which ERM algorithm they use as the base learner. With our adversarial ERM, the lower bound we prove holds for these algorithms even with a constant number of adversarial samples $m= O(1)$.

The adversarial ERM is however quite unnatural, and is explicitly designed to force a failure of majority voters. We therefore also consider a more natural ERM strategy, in which a uniformly random \emph{consistent} hypothesis from $\mathcal{H}$ is returned as the base learner. Here, a hypothesis is consistent if it correctly labels all training samples. We extend the previous lower bound to this case as well, still needing only a constant number of adversarial samples for the optimal algorithms mentioned above.

We now formally state our first lower bound.

\begin{theorem}[Majority Voting Lower Bound for Fixed ERM]
    \label{thm:majority-voter-lb-fixed-erm}
    For any $d \geq 1$ and $n \geq cd$ for sufficiently large constant $c>0$, there exists an ERM $E$, a hypothesis class $\mathcal{H}$ of VC dimension $d$ over a finite domain $\mathcal{X}$, a target $h^\star \in \mathcal{H}$ and a distribution $\mathcal{D}$ over $\mathcal{X}$, such that for any majority voter with sub-samples of size at least $t$, there is a monotone adversary $\mathcal{A}$ that uses $n$ clean and $m=2n/t$ adversarial samples, such that the expected error of the majority voter on $S \sim \adv_{\mathcal{D}, h^\star, \mathcal{A}}(n, m)$ and base learner $E$ is $\Omega(d \log(n/d)/n)$.
\end{theorem}

\begin{proof}
Let $r = cn/\log(n/d)$ for a sufficiently large constant $c>0$ and assume $r \geq d$ (which is the case for $n \geq  cd$). We define the input domain $\mathcal{X} = \{x_1,\dots,x_r, y_1,\dots,y_{\binom{r}{d}}\}$. We think of each $y_i$, instead, as $y_T$, corresponding to a $d$-sized subset $T$ of $\{1,\dots,r\}$. The hypothesis class $\mathcal{H}$ contains every hypothesis predicting $1$ on precisely $d$ of the $r$ points in $\{x_1,\dots,x_r\}$ and $0$ elsewhere. Finally, $\mathcal{H}$ also contains the target hypothesis $h^\star$ which is the constant $0$ function on the entire domain $\mathcal{X}$. Note that the VC dimension of $\mathcal{H}$ is $d$. The distribution $\mathcal{D}$ is uniform over $\{x_1,\dots,x_r\}$.

We now define the adversary $\mathcal{A}(\mathcal{D}, h^\star)$. On a clean sample $C \sim \mathcal{D}^n$, it checks if there is a subset $T$ of $d$ indices in $\{1,\dots,r\}$ for which %
$x_i \notin C,\, \forall i \in T$. If so, it picks an arbitrary such $T$ and outputs $m$ copies of the sample $y_T$ (recall each $y_T$ corresponds to a $d$-sized subset of $\{1,\dots,r\}$). Otherwise, it simply adds $m$ copies of the sample $x_1$.

We next define an ERM. On any labeled sample $S$, it checks if $S$ contains any point $y_T$. If so, it obtains the $d$-sized subset $T$ of $\{1,\dots,r\}$. If $S$ does not contain any $x_i$ for $i \in T$, it then outputs the hypothesis in $\mathcal{H}$ predicting $1$ on $x_i$ for every $i \in T$ and $0$ elsewhere. In all other cases, it outputs $h^\star$. Note that this is a valid ERM as its output hypothesis correctly predicts the label of all samples in its input.

We finally lower bound the expected error of any majority voter that is given input $S \sim \adv_{\mathcal{D}, h^\star, \mathcal{A}}(n,m)$, when $m$ is chosen to be sufficiently large as a function of $n$ and the majority voter's sub-sample size $t$. Here, we observe that if more than half of the subsets $S_1,\dots,S_k$ constructed by the majority voter from $S$ contain at least one copy of an adversarial sample $y_T$, corresponding to a set $T \subset \{1,\dots,r\}$, then the majority voter predicts $1$ on $x_i$ for every $i \in T$, resulting in an error of $|T|/r = \Omega(d \log(n/d)/n)$. We thus show that this event happens with constant probability. Recall from \Cref{def:majority-voter} that the lists $L_i$ of indices that determine the subsets $S_1,\dots,S_k$ are determined by the majority voter as a function of the size of the input set $S$ alone ($n+m$ in our case), and hence, the lists are independent of the random shuffling in $\adv_{\mathcal{D}, h^\star, \mathcal{A}}(n, m)$. So, if the adversary added $m$ adversarial samples $y_T$ to the initial input, then the probability that a sub-sample $S_i$ contains no adversarial samples is $\binom{n}{t}/\binom{n+m}{t} \leq (n/(n+m))^t \leq \exp(-mt/(n+m))$. For $m \geq 2n/t$, this is at most $e^{-1}$. Thus, conditioned on the initial clean sample $C \sim \mathcal{D}^n$ satisfying that there are at least $d$ indices in $\{1,\dots,r\}$ for which $x_i \notin C, \, \forall i \in T$, we get that the expected number of sub-samples $S_i$ with no adversarial sample is at most $k/e$. By Markov's inequality, we get that with a constant probability, more than $k/2$ sub-samples contain an adversarial sample $y_T$ and thus the error is $\Omega(d \log(n/d)/n)$. Therefore, all that remains to show is that there are at least at least $d$ indices in $\{1,\dots,r\}$ for which $x_i \notin C, \, \forall i \in T$ with constant probability. This follows immediately from our choice of $r$ and Lemma~\ref{lem:coupon}.
\end{proof}

We next consider the ERM strategy in which a uniformly random consistent hypothesis from $\mathcal{H}$ is returned on a given input sample. Denote this ERM by $\randERM$.

\begin{theorem}[Majority Voting Lower Bound for Random ERM]
    \label{thm:majority-voter-lb-random-erm}
    For any $d \geq 1$ and $n \geq cd$ for sufficiently large constant $c>0$, 
    there exists a hypothesis class $\mathcal{H}$ of VC dimension $d$ over a finite input domain $\mathcal{X}$, a target $h^\star \in \mathcal{H}$ and a distribution $\mathcal{D}$ over $\mathcal{X}$, such that for any majority voter with sub-samples of size at least $t$, there is a monotone adversary $\mathcal{A}$ that uses $n$ clean and $m=2n/t$ adversarial samples, such that the expected error of the majority voter on $S \sim \adv_{\mathcal{D}, h^\star, \mathcal{A}}(n, m)$ and base learner $\randERM$ is $\Omega(d \log(n/d)/n)$.
\end{theorem}

\begin{proof}
Our proof extends the proof of Theorem~\ref{thm:majority-voter-lb-fixed-erm} and we only explain the modifications. We thus assume that the reader has read the proof of Theorem~\ref{thm:majority-voter-lb-fixed-erm}.

First, we expand the universe by adding 1000 additional points $\{z_1,\dots,z_{1000}\}$. The target hypothesis is still the constant $0$ function. For each subset $T$ of $d$ indices among $\{1,\dots,r\}$, we will first specify a template hypothesis $h_T$ that determines labels on all points except $z_1,\dots,z_{1000}$. Concretely, $h_T(x_i) = 1$ if $i \in T$, and $h_T(x_i)=0$ otherwise. We set $h_T(y_T)=0$ for the $y_T$ that corresponds to $T$ (see the previous proof), and set $h(y_{T'})=1$ otherwise. 
Finally, we add $1000$ near-identical copies of $h_T$ to $\mathcal{H}$. The copy $h_{T,j}$ takes the same values as $h_T$ on $\{x_1,\dots,x_r,y_1,\dots,y_{\binom{r}{d}}\}$. On the points $z_1,\dots,z_{1000}$ it takes the value $1$ on $z_j$ and $0$ elsewhere. In all, $\mathcal{H}$ comprises of all the $h_{T,j}$ functions across all $T$ and $j$, together with $h^\star$. Observe that the VC dimension of $\mathcal{H}$ is at most $d+2$. 

We still consider the distribution $\mathcal{D}$ to be uniform over $\{x_1,\dots,x_r\}$, and employ the same adversarial strategy as in the proof of \Cref{thm:majority-voter-lb-fixed-erm}. %
Following the analysis there further, %
there is a constant probability that more than $51k/100$ of the sub-samples $S_1,\dots,S_k$ contain a copy of an adversarial sample $y_T$ with ground-truth label $0$; so, we condition on this event. Since $h_{T'}(y_T)=1$ for any $T' \neq T$, the only hypotheses in $\mathcal{H}$ consistent with these sub-samples are $h^\star$ and $h_{T,1},\dots,h_{T,1000}$. Thus, for each of these (at least $51k/100$ many) sub-samples, $\randERM$ returns a hypothesis predicting $1$ on all of $T$ with probability $1000/1001$. In particular, if we consider the first $51k/100$ of these sub-samples, the expected number of hypotheses returned by $\randERM$ that predict $1$ on all of $T$ is at least $510k/1001$. So, by Markov's inequality, with a constant probability, we have that at least $k/2$ many of the returned hypotheses predict $1$ on all of $T$, meaning that the expected error of the majority vote is at least $\Omega(d\log(n/d)/n)$ as required.
\end{proof}

\subsection{Lower Bound against One-Inclusion Graph Algorithms}
\label{sec:oig-lower-bounds}

We now turn to proving lower bounds for OIG algorithms. Since the predictions made by an OIG algorithm depend critically on the orientation of the edges chosen by the algorithm, we need to define which orientation strategy we will prove lower bounds against. In the standard i.i.d.\ setting, any orientation with maximum out-degree of $t$ gives an expected error of $O(t/n)$ \citep{haussler1994predicting}. Combined with \Cref{thm:oig-outdegree}, it follows that for a class of VC dimension $d$, we can compute an orientation with max out-degree $d$, and thus obtain an optimal expected error of $O(d/n)$.

In light of the above, we consider the following natural orientation strategy: for a one-inclusion graph $\mathcal{G}$ with vertex set $V$, let $\tau$ be the smallest achievable maximum out-degree, i.e.
$$\tau = \min_\sigma \max_{v \in V} \outd_\sigma(v),$$
where $\outd_\sigma(v)$ denotes the out-degree of $v$ under the orientation $\sigma$. We consider the orientation strategy $\OSrand$ that picks a uniformly random orientation $\sigma$ among all orientations that satisfy $\max_{v \in V} \outd_\sigma(v)=\tau$, i.e.\ an optimal max out-degree. For this orientation strategy, we show:

\begin{theorem}[Lower Bound for OIG]
    \label{thm:oig-lower-bound}
    For any $n \geq c$ for sufficiently large constant $c>0$, there exists a hypothesis class $\mathcal{H}$ of VC dimension $1$ over a finite input domain $\mathcal{X}$, a target $h^\star \in \mathcal{H}$ and a distribution $\mathcal{D}$ over $\mathcal{X}$, such that for the OIG algorithm with orientation strategy $\OSrand$, there is a monotone adversary that uses $n$ clearn and $m=n$ adversarial samples, such that the expected error of the OIG algorithm is at least $1/4$. Furthermore, for any $1 \leq k \leq c^{-1}n$, there is a monotone adversary that uses $n$ clean and $m=c n/\log(n/k)$ adversarial samples, such that the expected error of the OIG algorithm with orientation strategy $\OSrand$ is $\Omega(k \log(n/k)/n)$.
\end{theorem}
Choosing e.g., $k=n^{1-\eps}$ for an arbitrarily small constant $\eps > 0$ gives an adversary that uses $O(n/\log(n))$ adversarial samples and causes the OIG algorithm with orientation strategy $\OSrand$ to suffer an expected error of $\Omega(n^{-\eps})$.

\begin{proof}
We start by presenting the lower bound with an adversary that uses $n$ adversarial samples and then generalize it to $c n/\log(n/k)$ samples.

Consider the domain $\mathcal{X} = \{x_1,\dots,x_{2n}, y_{1}, \dots y_{2n}\}$. Let $\mathcal{H}$ be the hypothesis class that for every $i \in \{1,\dots,2n\}$ has a hypothesis $h_{i}$ with $h_{i}(x_j) = h_i(y_j) = 1$ if $j=i$ and $0$ otherwise. We let the target hypothesis $h^\star$ be the constant 0 function and add it to $\mathcal{H}$ as well. The VC dimension of this hypothesis class is $1$. We consider the distribution $\mathcal{D}$ that is uniform over $\{x_1,\dots,x_{2n}\}$.

We now define the adversary $\mathcal{A}(\mathcal{D}, h^\star)$. On a clean sample $C \sim \mathcal{D}^n$, it picks all $x_i \in C$ and adds in the corresponding point $y_i$.

To analyze the expected error of the OIG algorithm with orientation strategy $\OSrand$ on the above distribution and adversary, assume the algorithm receives a training sample $S \sim \adv_{\mathcal{D}, h^\star, \mathcal{A}}(n,m)$ and needs to make a prediction on a fresh point $x_j \notin S$ (if $x_j \in S$, the OIG algorithm will be correct on $x_j$). The event $x_j \notin S$ happens with probability at least $1/2$ since $\mathcal{D}$ is uniform over $2n$ points. Conditioned on this event, consider all the vertices $V$ in the graph $\mathcal{G}$ constructed by the algorithm, which correspond to projections of all the hypotheses in $\mathcal{H}$ on the (unlabeled) points in $S \cup \{x_j\}$. 
For any $h_i \in \mathcal{H}$, observe that if $x_i \in S$, then $h_i$ predicts $1$ on at least two points in $S$, namely $x_i$ and $y_i$. This implies that there is no edge between the all-0 vertex and the vertex in $V$ corresponding to $h_i$. Similarly, the vertex corresponding to $h_i$ also does not share an edge with any vertex corresponding to $h_\ell$, where $\ell \neq i$ and $x_{\ell} \in S$. Next, consider the hypothesis $h_j$: since $x_j \notin S$, $h_j$ predicts $1$ on precisely $x_j$, and $0$ on all other points in $S$. There is thus an edge between the vertex corresponding to $h_j$ and the all-0 vertex, but no edge between the vertex corresponding to $h_j$ and any vertex corresponding to $h_\ell$, where $\ell \neq j$ and $x_{\ell} \in S$. Finally, for any $h_i$ with $x_i \notin S$ and $i \neq j$, observe that $h_i$ predicts $0$ on all of $S \cup \{x_j\}$, hence collapsing to the all-0 vertex. We have thus argued that the only edge in the graph is the edge between the all-0 vertex and the vertex corresponding to $h_j$ --- every other vertex in the graph is isolated. An orientation $\sigma$ simply has to orient this single edge, and for either of the two ways that it can orient it, the max out-degree of a vertex in the graph is 1. Since $\sigma$ is chosen uniformly at random among all orientations with smallest max out-degree, with probability $1/2$, this edge will be oriented away from the all-0 vertex, in which case the algorithm mispredicts the label on $x_j$. We conclude that the expected error is at least $1/4$.

To extend the above lower bound to $r=cn/\log(n/k)$ adversarial samples, we consider the domain $\mathcal{X} = \{x_1,\dots,x_{r}, y_{1}, \dots y_{r}\}$ with $\mathcal{D}$ uniform over $\{x_1,\dots,x_r\}$. We let $\mathcal{H}$ be as above, so that $\mathcal{H}$ contains a hypothesis $h_i$ for $i=1,\dots,r$ making predictions $h_i(x_j)=h_i(y_j)=1$ if $j=i$ and $0$ otherwise. We let $h^\star$ be all-0.

We now define the adversary $\mathcal{A}(\mathcal{D}, h^\star)$. On a clean sample $C \sim \mathcal{D}^n$ it picks all $i \in \{1,\dots,r \}$ such that $x_i$ appears at least once in $C$ and adds in the point $y_i$. This specifies the draw of $S \sim \adv_{\mathcal{D}, h^\star, \mathcal{A}}(n,m)$ for $m=r$. %

By the same arguments as above, we see that if the OIG algorithm is asked to predict the label of a point $x_j \notin S$, then the one-inclusion graph has precisely one edge in it connecting the all-0 vertex and the vertex corresponding to the hypothesis $h_j$. We thus have that the OIG algorithm errs with probability at least $1/2$ on such $x_j$. By the coupon collector argument in \Cref{lem:coupon}, we have that with probability at least $1/2$ over the draw of $S$, there are at least $k$ points $x_i$ that do not appear in $S$. Each of these may be drawn from $\mathcal{D}$ with probability $1/r$, and causes the OIG algorithm to incur an error $1/2$. We conclude that the expected error of the OIG algorithm is thus $\Omega(k/r) = \Omega(k \log(n/k)/n)$.
\end{proof}

\section*{Acknowledgements}
CP was supported by Gregory Valiant's and Moses Charikar's Simons Investigator Awards, and a Google PhD Fellowship. KGL is funded by the European Union (ERC, TUCLA, 101125203).
AS is supported in part by ARO award
W911NF-21-1-0328, the Simons Foundation, NSF award DMS-2031883, a DARPA AIQ award, and an NSF
FODSI Postdoctoral Fellowship.
Views and opinions expressed are however those of the author(s) only and do not necessarily reflect those of the European Union or the European Research Council. Neither the European Union nor the granting authority can be held responsible for them. 
AS would further like to acknowledge various fruitful conversations with Jason Gaitonde, Surbhi Goel and Noah Golowich. 

\bibliographystyle{alpha}
\bibliography{references}

\appendix

\section{Proofs of Auxiliary Results}
\label{sec:aux}

\begin{proof}[Proof of Lemma~\ref{lem:coupon}]
Recall that $r= cn /\log(n/d)$ for sufficiently large constant $c>0$. Define an indicator $X_i$ for each $x_i$, taking the value $1$ if $x_i \notin S$. Then $\Pr[X_i=1] = (1-1/r)^n$. For two distinct $i \neq j$ we have $\Pr[X_j=1 \mid X_i=1] = (1-1/(r-1))^n$. Hence $\Ex[X_i X_j] = (1-1/r)^n(1-1/(r-1))^n$. We have $\Ex[\sum_i X_i] = r(1-1/r)^n$ and 
\begin{align*}
    \Ex\left[\left(\sum_i (X_i - (1-1/r)^{n})\right)^2\right] &= \\
    \sum_i \left(\Ex[X^2_i]-(1-1/r)^{2n}\right) + \sum_{i \neq j}\left(\Ex[X_i X_j] - (1-1/r)^{2n} \right)
    &= \\
    \sum_i \left((1-1/r)^n-(1-1/r)^{2n}\right) + \sum_{i \neq j}\left((1-1/r)^n(1-1/(r-1))^n - (1-1/r)^{2n}\right) &= \\
    r (1-1/r)^n-r(1-1/r)^{2n} + \sum_{i \neq j}(1-1/r)^{2n}\left[\left(\frac{(r-1)^2}{r(r-2)} \right)^n - 1\right] &\leq \\
    r(1-1/r)^n + r^2(1-1/r)^{2n}\left(\left(1 + \frac{1}{r^2 - 2r} \right)^n-1\right) &\leq \\
    r(1-1/r)^n + r^2(1-1/r)^{2n}\left(\exp\left(\frac{n}{r^2 - 2r} \right)-1\right). \tag{using $1+x \le e^x$} 
\end{align*}
For $c$ sufficiently large, we have $n \leq (1/2)(r^2-2r)$ and thus $\exp(n/(r^2-2r)) \leq 1+ 2n/(r^2-2r)$ (using that $\exp(x) \leq 1 + 2x$ for $0 \leq x \leq 1/2$); we conclude
\begin{align*}
    \Ex\left[\left(\sum_i (X_i - (1-1/r)^{n})\right)^2\right] &\leq \\
    r(1-1/r)^n + r^2(1-1/r)^{2n} \cdot \frac{2n}{r^2-2r} &\leq \\
    r(1-1/r)^n + r^2(1-1/r)^{2n}/16.
\end{align*}
Here again, the last inequality holds for $c$ sufficiently large. By Chebyshev's inequality, we have that
\begin{align*}
    \Pr\left[\sum_i X_i \leq (1/2)r(1-1/r)^n\right] &\leq \frac{r(1-1/r)^n + r^2(1-1/r)^{2n}/16}{(1/4)r^2(1-1/r)^{2n}} \\
    &= \frac{4}{r(1-1/r)^n} + 1/4 \\
    &\leq \exp(n/r)/r + 1/4 \\
    &= \exp(\log(n/d) c^{-1})/r + 1/4\\
    &\leq 1/2.
\end{align*}
where the last inequality holds for large enough constant $c>0$. We conclude by noting that $d \le (1/2)r(1-1/r)^n$ under the assumptions of the lemma, and for large enough $c$.
\end{proof}

\end{document}